\documentclass[sigconf]{aamas} 

\usepackage{balance} 

\usepackage{url}            
\usepackage{booktabs}       
\usepackage{amsfonts}       
\usepackage{nicefrac}       
\usepackage{microtype}      
\usepackage{xcolor}         

\usepackage{amsmath}
\usepackage{mathtools}
\usepackage{amsthm}
\usepackage{url}

\usepackage{algorithm}
\usepackage{algorithmic}
\usepackage{verbatim}

\usepackage{amsmath}
\usepackage{mathtools}
\usepackage{amsthm}
\usepackage{microtype}
\usepackage{graphicx}
\usepackage{subfigure}
\usepackage{booktabs}

\usepackage[capitalize,noabbrev]{cleveref}
\definecolor{ForestGreen}{rgb}{0.13, 0.55, 0.13}
\theoremstyle{plain}
\newtheorem{theorem}{Theorem}[section]

\newtheorem{corollary}[theorem]{Corollary}
\theoremstyle{definition}

\theoremstyle{remark}

\setcopyright{ifaamas}
\acmConference[AAMAS '24]{Proc.\@ of the 23rd International Conference
on Autonomous Agents and Multiagent Systems (AAMAS 2024)}{May 6 -- 10, 2024}
{Auckland, New Zealand}{N.~Alechina, V.~Dignum, M.~Dastani, J.S.~Sichman (eds.)}
\copyrightyear{2024}
\acmYear{2024}
\acmDOI{}
\acmPrice{}
\acmISBN{}

\acmSubmissionID{180}

\title[AAMAS-2024 Formatting Instructions]{Auto-Encoding Adversarial Imitation Learning}

\author{Kaifeng Zhang\textsuperscript{1}, Rui Zhao\textsuperscript{2}, Ziming Zhang\textsuperscript{3} and Yang Gao$^{\dagger}$\textsuperscript{4,5,1}}
\affiliation{
  \institution{\textsuperscript{1}Shanghai Qi Zhi Institute, \textsuperscript{2}Tencent AI Lab,  \textsuperscript{3}Worcester Polytechnic Institute \\ \textsuperscript{4}Tsinghua University, \textsuperscript{5}Shanghai Artificial Intelligence Laboratory}
  \country{$^{\dagger}$Corresponding Author}}
\email{zhangkf@sqz.ac.cn, rui.zhao.ml@gmail.com, zzhang15@wpi.edu, gaoyangiiis@mail.tsinghua.edu.cn}

\begin{abstract}
Reinforcement learning (RL) provides a powerful framework for decision-making, but its application in practice often requires a carefully designed reward function. Adversarial Imitation Learning (AIL) sheds light on automatic policy acquisition without access to the reward signal from the environment. In this work, we propose Auto-Encoding Adversarial Imitation Learning (AEAIL), a robust and scalable AIL framework. To induce expert policies from demonstrations, AEAIL utilizes the reconstruction error of an auto-encoder as a reward signal, which provides more information for optimizing policies than the prior discriminator-based ones. Subsequently, we use the derived objective functions to train the auto-encoder and the agent policy. Experiments show that our AEAIL performs superior compared to state-of-the-art methods on both state and image based environments. More importantly, AEAIL shows much better robustness when the expert demonstrations are noisy. 
\end{abstract}

\keywords{Reinforcement Learning, Adversarial Imitation Learning, Auto-Encoding}

\newcommand{\BibTeX}{\rm B\kern-.05em{\sc i\kern-.025em b}\kern-.08em\TeX}

\begin{document}


\pagestyle{fancy}
\fancyhead{}


\maketitle 


\section{Introduction}
Reinforcement learning (RL) provides a powerful framework for automated decision-making.
However, RL still requires significantly engineered reward functions for good practical performance. 
Imitation learning offers the instruments to learn policies directly from the demonstrations, without an explicit reward function. 
It enables the agents to learn to solve tasks from expert demonstrations, such as helicopter control \citep{NIPS2006, NIPS2007, ISER_helicopter2004, ICML_helicopter2008, helicopter_2008AARCH, Abbeel_2010}, robot navigation \citep{Ratliff2006, Abbeel2008, Ziebart2008, Ziebart2010}, and building controls \citep{Barrett2015}. 

The goal of imitation learning is to induce the expert policies from expert demonstrations without access to the reward signal from the environment. 
We divide these methods into three broad categories: Behavioral Cloning (BC), Inverse Reinforcement Learning (IRL), and Adversarial Imitation Learning (AIL). AIL induces expert policies by minimizing the distribution distance between expert samples and agent policy rollouts. 
Prior AIL methods model the reward function as a discriminator to learn the mapping from the state-action pair to a scalar value, i.e., reward \citep{GAIL, fGAIL, AIRL}. 
However, the discriminator in the AIL framework would easily find the differences between expert samples and agent-generated ones, even though some differences are minor. Therefore, the discriminator-based reward function would yield a sparse reward signal to the agent. 
Consequently, how to make AIL robust and efficient to use is still subject to research.

Our AEAIL is an instance of AIL by formulating the reward function as an auto-encoder. Since auto-encoder reconstruct the full state-action pairs, unlike traditional discriminator based AIL, our method will not overfit to the minor differences between expert samples and generated samples. In many cases, our reward signal provides richer feedback to the policy training process. Thus, our new method achieves better performance on a wide range of tasks. 

Our contributions are three-fold: 
\begin{itemize}
\item We propose the Auto-Encoding Adversarial Imitation Learning (AEAIL) architecture. It models the reward function as the reconstruction error of an auto-encoder. It focuses on the full-scale differences between the expert and generated samples, and less susceptible to the discriminator attending to minor differences. 

\item Experiments show that our proposed AEAIL achieves the best overall performance on many environments compared to other state-of-the-art baselines. 

\item Empirically, we justify that our AEAIL works under a wide range of distribution divergences and auto-encoders. And the major contributing factor of our method is the encoding-decoding process rather than the specific divergence or auto-encoder. 
\end{itemize}

\section{Related work}
\subsection{Imitation learning}

\textbf{Inverse Reinforcement Learning (IRL)} IRL learns a reward function from the expert demonstrations, which assumes the expert performs optimal behaviors. One category of IRL uses a hand-crafted similarity function to estimate the rewards \citep{boularias2011relative, klein2013cascaded, piot2016bridging}. The idea of these methods is inducing the expert policy by minimizing the distance between the state action distributions of the expert and sampled trajectories. Primal Wasserstein Imitation Learning (PWIL) \citep{PWIL} is one such algorithm, which utilizes the upper bound of Wasserstein distance's primal form as the optimization objective. The advantage of these methods is that they are relatively more robust and less demonstration dependent than AIL methods. However, the performance of these methods heavily depends on the similarity measurement, and therefore, it varies greatly on different tasks. Compared to PWIL, our method achieves superior performance via automatically acquiring a similarity measurement. 
Random Expert Distillation (RED) \citep{RED} uses random network distillation to compute the rewards for imitation learning. The error between the predictor and the random target network could guide the policy to mimic the expert behaviors. Meanwhile, auto-encoder could be used in place of the random network distillation. In comparison, RED uses a fixed auto-encoder while our AEAIL utilizes the AE in an adversarial manner. ValueDice \citep{ValueDice} utilizes a critic based reward function which is optimized during training the policy and critic. However, the performance of ValueDice is sensitive to the expert data used. IQ-Learn \citep{IQ-Learn} is a dynamics-aware IL method which avoids adversarial training by learning a single Q-function, implicitly representing both the reward and the policy.

\textbf{Adversarial Imitation Learning (AIL)} AIL methods such as GAIL, DAC, $f$-GAIL, EAIRL, FAIRL \citep{GAIL, DAC, fGAIL, EAIRL, FAIRL} formulates the learned reward function as a discriminator that learns to differentiate expert transitions from non-expert ones. Among these methods, GAIL \citep{GAIL} considers the Jensen-Shannon divergence. DAC \citep{DAC} extends GAIL to the off-policy setting and significantly improves the sample efficiency of adversarial imitation learning. Furthermore, $f$-divergence is utilized in $f$-GAIL \citep{fGAIL}, which is considered improving sample-efficiency. Recently, FAIRL utilizes the forward KL divergence \citep{FAIRL} and achieves great performance, but it is still not robust and efficient enough. These methods rely heavily on a carefully tuned discriminator-based reward function. It might focus on the minor differences between the expert and the generated samples and thus gives a sparse reward signal to the agent. In comparison, our auto-encoder-based reward function learns the full-scale differences between the expert and generated samples.

\subsection{Imitation Learning with Imperfection} 

The robustness of adversarial imitation learning is questionable with imperfection in observations \citep{Third-Person, berseth2020visual}, actions, transition models \citep{State-only, christiano2016transfer}, expert demonstrations \citep{brown2019extrapolating, shiarlis2016inverse, jing2020reinforcement} and their combination \citep{kim2020domain}. Prior robust imitation learning methods require the demonstrations to be annotated with confidence scores \citep{ImitationImperfection, brown2019extrapolating, grollman2012robot}. However, these annotations are rather expensive. Compared to them, our auto-encoder-based reward function won't easily overfit to the noisy feature. Therefore, our AEAIL is relatively straightforward and robust to the noisy expert data and does not require any annotated data. 

\subsection{Auto-Encoding based GANs}

Auto-encoders have been successfully applied to improve the training stability and modes capturing in GANs. We classify auto-encoding based GANs into three categories: (1) utilizing an auto-encoder as the discriminator such as energy-based GANs \citep{EBGAN} and boundary-equilibrium GANs \citep{BEGAN}; (2) using a denoising auto-encoder to derive an auxiliary loss for the generator \citep{warde2016improving}; (3) combining variational auto-encoder and GANs to generate both vivid and diverse samples by balancing the objective of reconstructing the training data and confusing the discriminator \citep{VAE-GAN}. Our method AEAIL takes inspirations from EBGAN \citep{EBGAN}, utilizing the auto-encoder in the reward function.

\section{Background}

\textbf{Inverse Reinforcement Learning (IRL)} IRL extracts the policy towards mimicking expert behaviors with the help of a learned reward function. On the other hand, AIL directly induces the expert policies by minimizing the distribution distance between expert demonstrations and agent-generated samples. 

\textbf{Adversarial Imitation Learning (AIL)} Assume that we are given a set of expert demonstrations, AIL involves a discriminator which tries to distinguish the expert and policy generated samples and a policy that is optimized to confuse the discriminator accordingly. Ideally, the rollouts from the agent will be indistinguishable from the expert demonstrations in the end. The discriminator cannot classify the expert samples and policy-generated ones any more at the end of the training process. Consequently, this kind of discriminator-based reward function can be seen as a pseudo-reward signal, which is only used for optimizing policies during training. 

\textbf{Wasserstein Divergence} One popular distribution divergence for measuring the distance between policy samples and expert samples is the Wasserstein divergence \citep{WGAN}:

\begin{align}\label{w-distance}
d(p, q) = \sup_{|\phi|_L \leq 1} {E}_{x \sim p}[\phi(x)] - {E}_{y \sim q}[\phi(y)], 
\end{align}

where $p$ and $q$ are two probability distributions, $\phi$ is some function that is 1-Lipschitz continuous. Equation \ref{w-distance} measures the distance between these two distributions. Intuitively, we can view it as the minimum cost of turning the piles from distribution $p$ into the holes in distribution $q$.

\begin{figure*}
\begin{center} 
\includegraphics[width=4.5in]{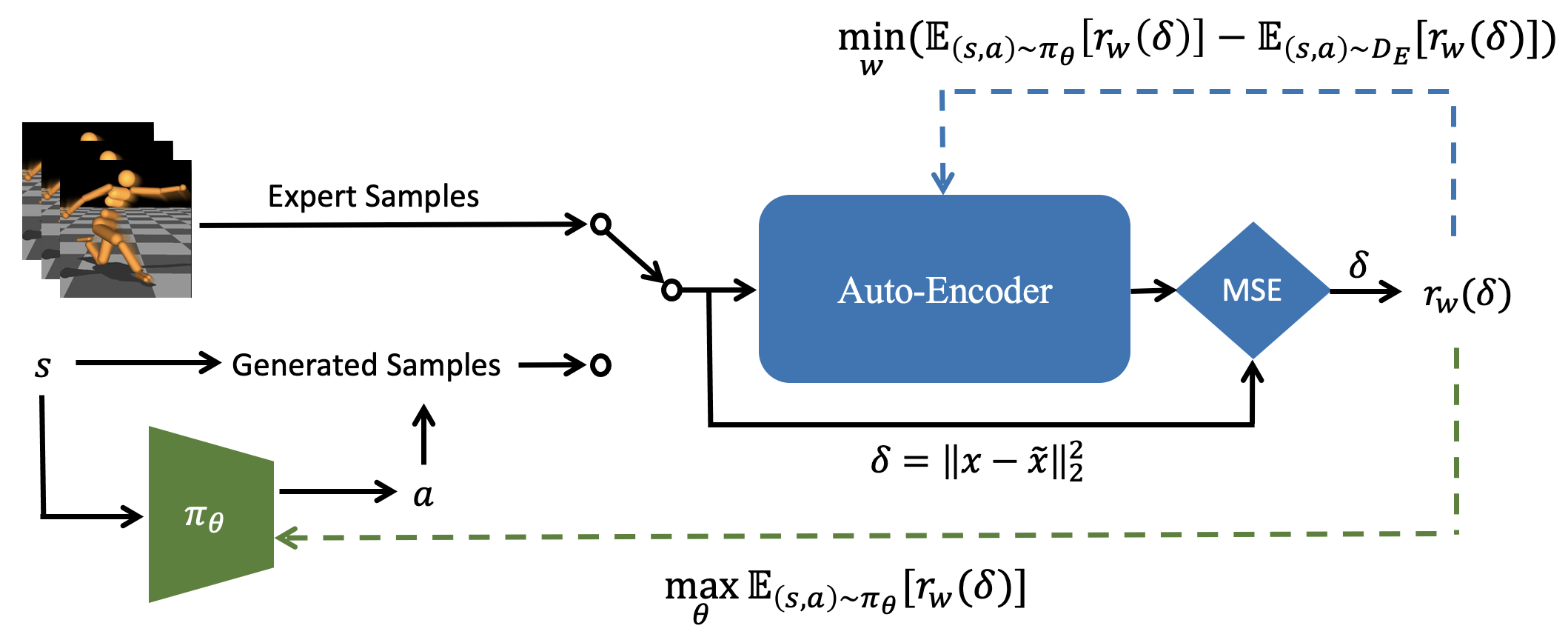} 
\caption{Training framework of our AEAIL. The auto-encoder computes the reconstruction error for these two mini-batches of data examples and optimizes the derived objectives. Therefore, a surrogate reward function $r_w(\delta)$ provides the signal to the agent. The agent can induce the expert policies by iteratively training the auto-encoder and the agent policy.}
\label{AEIRL} 
\end{center} 
\skip -0.1in
\end{figure*}

\section{Methodology}
\subsection{Overview}\label{motivation}

The reward function in AIL methods is a discriminator, which attempts to assign high values to the regions near the expert demonstrations and assign low values to the other areas. However, the discriminator would easily overfit to the minor differences between expert data and policy-generated samples. Consider the CartPole balancing task, the goal of this task is to keep the pole balancing via moving the cart left or right. The state of the environment is a feature consisting of position, angle, angle's rate, and cart velocity while the action is moving left or right. Here, we assume that all the expert states' velocity is $2$, for example. When the generated states' velocity is $1.9$, and other dimensions of state-action pair are the same as the expert's, the discriminator of GAIL would still give a low reward on these generated state-action samples. However, the policy rollouts may perform very well towards mimicking the expert's behaviors. In other words, the reward function in GAIL on this example overfits to the minor differences between the expert and the sampled data and leads the policy to miss the underlying goal of the expert. 

In this paper, we propose an auto-encoder-based reward function for adversarial imitation learning. It utilizes the reconstruction error of the state action pairs to compute the reward function. The reconstruction error based reward signal force the discriminator to consider all information in the state-action tuple, rather than focusing on the minor differences. 
Therefore, this reward signal wouldn't lead to overconfidence in distinguishing the expert and the generated samples. Recall the CartPole example, the mean square error between states' velocity $1.9$ and $2$ is very small. And it could still feed a good reward to the agent under this situation. Therefore, intuitively our proposed reward signal can better focus on the full-scale differences between the expert and generated state action pairs, yielding more informative rewards and thus improving the policy learning.

\subsection{Our method}\label{Method}

Our approach is to minimize the distance between the state action distribution of the policy $\pi_\theta$ and that of the expert demonstrations $D_E$. 

The objective formulation we used in our method is Wasserstein divergence:
\begin{align}\label{objective}
d(\pi_E, \pi_\theta) = \sup_{|r_w|_L \leq K} {E}_{\pi_E}[r_w(s,a)] - {E}_{\pi_\theta}[r_w(s,a)], 
\end{align}

where the reward function network's parameters are denoted as $w$ and the policy network's parameters are represented as $\theta$. 
Minimizing this distance will induce the expert policy from expert demonstrations. Therefore, the optimization of the policy $\pi_\theta$ and the reward function $r_w(s,a)$ forms a bi-level optimization problem, which can be formally defined as:
\begin{align}\label{bi-level}
\min_{\pi_\theta} \max_{r_w}  \left( {E}_{(s,a) \sim D_E}[r_w(s,a)] - {E}_{(s,a)\sim \pi_{\theta}}[r_w(s,a)] \right).
\end{align}
Here, we clip the weights of the reward function to keep it K-Lipschitz. Optimizing Equation \ref{bi-level} leads to an adversarial formulation for imitation learning. The outer level minimization with respect to the policy leads to a learned policy that is close to the expert. The inner level maximization recovers a reward function that attributes higher values to regions close to the expert data and penalizes all other areas.

In our method, we use an auto-encoder based surrogate pseudo-reward function instead, which is defined as:
\begin{align}\label{RFunction}
r_w(s, a) = 1/(1+\text{AE}_w(s,a)),
\end{align}
where AE is the reconstruction error of an auto-encoder:
\begin{align}\label{AE}
\textstyle \text{AE}(x) = \Vert \text{Dec}\circ \text{Enc}(x) - x \Vert^2_2
\end{align}
Here, $x$ represents the state-action pairs. Equation \ref{AE} is the mean square error between the sampled state-action pairs and the reconstructed samples. 
This form of the reward signal uses the reconstruction error of an auto-encoder to score the state-action pairs in the trajectories. Equation \ref{RFunction} is a monotonically decreasing function over the reconstruction error of the auto-encoder. We give high rewards to the state-action pairs with low reconstruction errors of the auto-encoding process and vice versa. Section "Overview" \ref{motivation} motivates that this form of reward signal focuses more on the full-scale differences between the expert and generated samples, and it won't easily overfit to the noise of the expert data. 

Training the auto-encoder is an adversarial process considering the objective \ref{bi-level}, which is minimizing the reconstruction error for the expert samples while maximizing this error for generated samples. When combining Equation \ref{bi-level} and Equation \ref{RFunction}, we can obtain the training objective for the auto-encoder as:
\begin{align}\label{Loss}
\mathcal{L} = &{E}_{(s,a)\sim \pi_{\theta}}[r_w(s,a)] - {E}_{(s,a)\sim D_E}[r_w(s,a)]  \\
=& {E}_{(s,a)\sim \pi_{\theta}}[1/(1+\text{AE}_w(s,a))]\\
&-  {E}_{(s,a)\sim D_E}[1/(1+\text{AE}_w(s,a))].
\end{align}

With this adversarial objective, the auto-encoder learns to maximize the full-scale differences between the expert and the generated samples. As a result, it gives the agent a denser reward signal. Furthermore, the agent can also be more robust when facing noisy expert data due to the robust auto-encoding objective.

Figure \ref{AEIRL} depicts the architecture for AEAIL. The auto-encoder-based reward function takes either expert or the generated state-action pairs and estimates the reconstruction error based rewards accordingly. We iteratively optimize the auto-encoder and the agent policy under this adversarial training paradigm. 

\begin{algorithm}[htb]
   \caption{Auto-Encoding Adversarial Imitation Learning (AEAIL)}\label{Algorithm}
	\begin{algorithmic}
	    \STATE {\bfseries Input:} Initial parameters of policy, auto-encoder $\theta_0$, $w_0$; Expert trajectories $D_{E}$. 
        \REPEAT
		\STATE Sample state-action pairs $(s_i, a_i) \sim \pi_{\theta_i}$ and $(s_E, a_E) \sim D_{E}$ with same batch size. \\
		\STATE Update $w_i$ to $w_{i+1}$ by decreasing with the gradient:
		\begin{align*}
		    {E}_{(s_i,a_i)}[\nabla_{w_i} 1/(1 + \text{AE}_{w_i}(s_i, a_i))] -  {E}_{(s_E, a_E)}[\nabla_{w_i} 1/(1+\text{AE}_{w_i}(s, a))]
		\end{align*}
		\STATE Take a policy step from $\theta_i$ to $\theta_{i+1}$, using the TRPO update rule with the reward function $1/(1 + \text{AE}_{w_i}(s, a))$, and the objective function for TRPO is:
		\begin{align*}
		 {E}_{(s,a)}[- 1/(1 + \text{AE}_{w_i}(s, a))].
		\end{align*}
		\UNTIL{$\pi_{\theta}$ converges}
    \end{algorithmic}
\end{algorithm}

Algorithm \ref{Algorithm} depicts the pseudo-code for training AEAIL. The first step is to sample the state-action pairs from expert demonstrations $D_E$ and the trajectories sampled by the current policy with the same batch size. Then we update the auto-encoder by decreasing with the gradient with loss function Eq. \ref{Loss}. Finally, we update the policy assuming that the reward function is Eq. \ref{RFunction} which leads to an adversarial training paradigm. We repeat these steps until the approach converges. 

\subsection{Theoretical Properties}

\begin{theorem}\label{theorem}
Provided that $f(x)$ is K-Lipschitz, then the reward function formulation $1/(1+(x-f(x))^2)$ is 2m(K+1)-Lipschitz, where $m$ is a bounded constant. 
\end{theorem}

\begin{proof}
See in Appendix \ref{Appendix:proof-theorem}.
\end{proof}

Theorem \ref{theorem} shows that clipping the weights of the auto-encoder would make the reward function K-Lipschitz. Therefore, we can obtain Corollary \ref{corollary1} which shows that our AEAIL is actually minimizing a Wasserstein distance. 

\begin{corollary}\label{corollary1} (Our AEAIL is minimizing Wasserstein distance) the divergence for AEAIL:
\begin{align}\label{objective}
d(\pi_E, \pi_\theta) = \sup_{|r_w|_L \leq K} {E}_{\pi_E}[r_w(s,a)] - {E}_{\pi_\theta}[r_w(s,a)], 
\end{align}
is a Wasserstein distance. 
\end{corollary}

\begin{proof}
See in Appendix \ref{Appendix:proof-corollary}.
\end{proof}

In AEAIL, we utilize an auto-encoder in place of the discriminator for more empirical advantages. Corollary \ref{corollary2} shows that our AEAIL is equivalent to using a generalized discriminator. As a consequence, AEAIL shares the same convergence properties with GAIL under some assumptions. 

\begin{corollary}\label{corollary2} Our AEAIL shares a similar theoretical framework with GAIL, where the auto-encoder based reward function can be seen as a generalized discriminator. 
\citep{theory} investigates the theoretical properties of GAIL. We make the same set of assumptions for our auto-encoder based reward function as these for the discriminator in GAIL \citep{theory}:

(i) the auto-encoder based reward function has bounded gradient, i.e., $\exists C_r \in {R}$, such that 
\[\Vert \nabla_w r_w \Vert_{\infty,2} := \sqrt{\sum_i \Vert \frac{\partial r_w}{\partial w_i} \Vert_{\infty}^2} \leq C_r\]
(ii) the auto-encoder based reward function has L-Lipschitz continuous gradient, i.e.,  $\exists L_r \in {R}$, such that $\forall s \in \mathcal{S}, \forall a \in \mathcal{A}, \forall w_1, w_2 \in \mathcal{W}$,  then: 
\[\Vert \nabla_w r_{w_1}(s,a) - \nabla_w r_{w_2}(s,a) \Vert_2 \leq L_r \Vert w_1 - w_2 \Vert_2\]

(iii) the distribution divergence
\[
d_{\mathcal{F}}(P_r, P_{\theta}) = \sup_{r \in \mathcal{R}}{E}_{x\sim P_r}[r(x)] - {E}_{x \sim P_\theta}[r(x)] + \phi(r(x))
\]
is properly regularized to be strongly concave with respect to the reward function. Here, $\phi(r(x))$ is the regularizer.

These conditions are assumptions for theoretical convenience since they don't always hold for a multi-layer neural network-based discriminator in GAIL or our proposed auto-encoder. When these assumptions hold, our AEAIL shares the same theoretical properties as GAIL: it converges globally. 
\end{corollary}

\begin{figure*}[htb]
	\centering
	\begin{minipage}{0.99\linewidth}\centering
		\includegraphics[width=\linewidth]{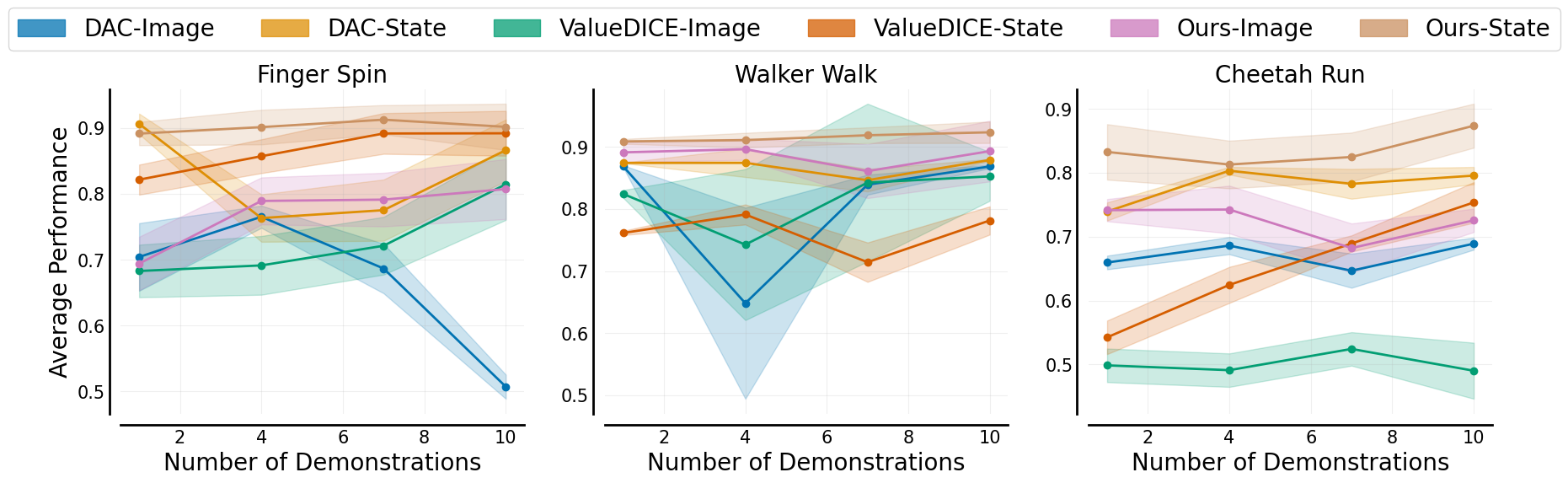}\\
	\end{minipage} 
	\caption{Mean and standard deviation return of the final policy performance over 10 rollouts and 3 seeds.}
	\label{lineplots}
\end{figure*}

\section{Experiments}

We conduct experiments on DMControl and MuJoCo-v2 environments, dividing them into two distinct phases. The first phase aim to assess whether our AEAIL could outperform other state-of-the-art imitation learning algorithms in both state and image-based environments. The second phase focuses on evaluating the robustness of AEAIL in the presence of noisy expert data, different divergence metrics, various auto-encoders, and different AE architectures. By addressing these two key aspects, we aim to comprehensively analyze the performance and versatility of our AEAIL algorithm.

\subsection{Phase 1: Comparison on state and image settings}

We provide a comprehensive comparison of DAC, ValueDICE, and our AEAIL in both state and image environment settings. To demonstrate the superior imitation capabilities of our AEAIL, we conduct experiments using the DrQ-v2 RL agent for a fair comparison. The number of expert demonstrations ranges from 1 to 10 trajectories for the experiments.

In terms of implementation, for the image-based setting, we employ an encoder for representation learning across the actor, critic, and discriminator. However, we only utilize the critic loss to train the encoder for improved representation learning. This approach aligns with the findings in ROT \cite{ROT}, which suggest that stopping the actor and discriminator gradients for the encoder leads to better imitation performance, particularly in image settings.

Figure \ref{lineplots} illustrates the scaled imitation performance of the baselines and our AEAIL in both state and image settings. Overall, the performance in the image setting is slightly lower than that in the state setting, particularly for ValueDICE. Our AEAIL demonstrates increasing imitation performance with a higher number of expert demonstrations, as more data proves beneficial for the adversarial training paradigm.

In the state setting, our AEAIL achieves relative improvements of $8.9\%$ and $4.2\%$ compared to DAC and ValueDICE, respectively, highlighting its superiority in imitation. Furthermore, in the image setting, our AEAIL achieves relative improvements of $15.7\%$ and $5.9\%$ compared to DAC and ValueDICE, respectively, indicating its greater promise in high-dimensional environments. Notably, our AEAIL exhibits enhanced robustness by focusing on the comprehensive differences between expert and generated rollouts, rather than easily overfitting to minor discrepancies between the expert and non-expert's. 
In addition, our AEAIL achieves overall expert performances of $90.1\%$ and $77\%$ in the state and image settings, respectively. It demonstrates remarkable success in solving the task in the state setting and achieves competitive results even in the image setting. 

\subsection{Phase 2: Empirical analysis}

\begin{figure*}[htb]
	\centering
	\begin{minipage}{0.3\linewidth}\centering
		\includegraphics[width=\linewidth]{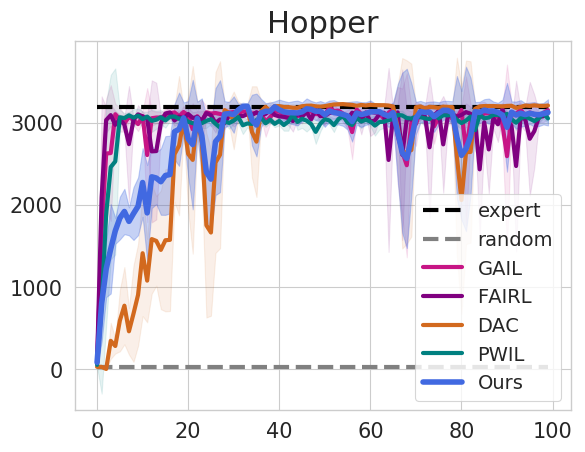}\\
	\end{minipage} 
	\begin{minipage}{0.3\linewidth}\centering
		\includegraphics[width=\linewidth]{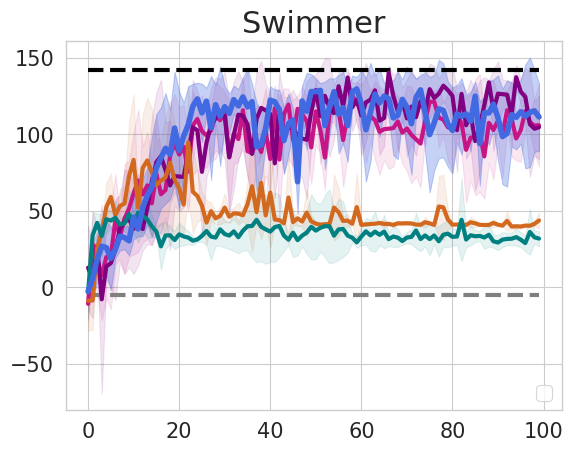}\\
	\end{minipage} 
	\begin{minipage}{0.3\linewidth}\centering
		\includegraphics[width=\linewidth]{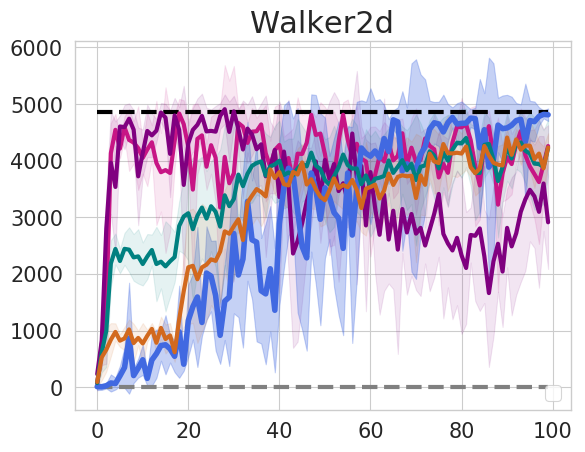}\\
	\end{minipage}
 
	\begin{minipage}{0.3\linewidth}\centering
		\includegraphics[width=\linewidth]{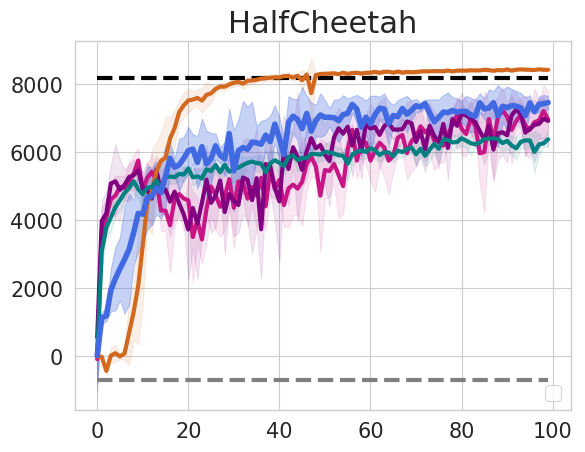}\\
	\end{minipage}
	\begin{minipage}{0.3\linewidth}\centering
		\includegraphics[width=\linewidth]{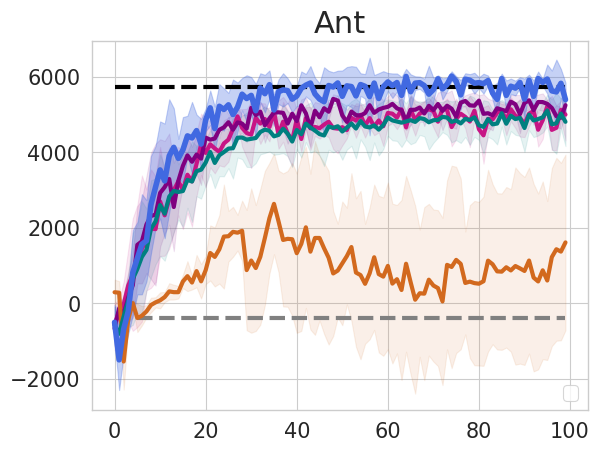}\\
	\end{minipage} 
	\begin{minipage}{0.3\linewidth}\centering
		\includegraphics[width=\linewidth]{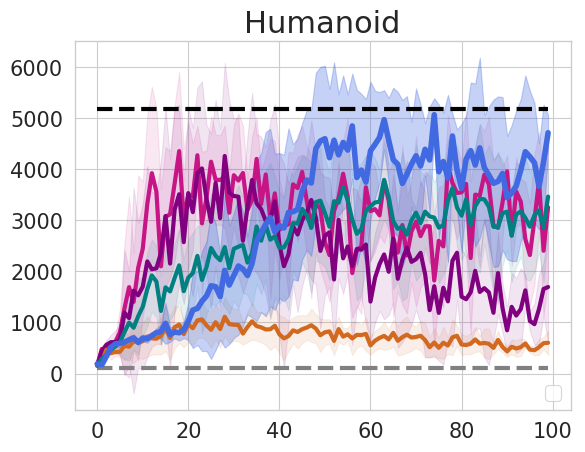}\\
	\end{minipage} 
	\caption{Mean and standard deviation return of the deterministic evaluation policy over 10 rollouts and 5 seeds, reported every 100k timesteps, which are learning from clean expert demonstrations.}
	\label{NonNoisyCurve}
\end{figure*}

\begin{table*}
\caption{Learned policy performance for different imitation learning ablations on clean expert data.}
\label{final_comparison_non_noisy}
\begin{center}
\begin{sc}
\begin{tabular}{lcccccr}
\toprule
Task &  Walker2d & Hopper & Swimmer & HalfCheetah & Ant & Humanoid   \\
\midrule

Expert &  $4865.1$  & $3194.9$ &$142.2$ & $8193.6$ & $5730.9$   & $5187.9$  \\

JSD &  $4259 \pm 369$   & $3152 \pm 102$ & $ 106 \pm 27$ & $6964 \pm 807$ & $4998 \pm 686$ & $ 3257\pm 1393$ \\

F-KLD  &  $2912 \pm 825$   & $ 3117\pm 75$ & $ 105\pm 21$ & $ 6917\pm 267$ & $ 5247 \pm 259$ & $1695\pm 1098$ 	\\

GOT  & $4226 \pm 501$   & $3051 \pm 51$  & $32 \pm 5$ & $6376 \pm229$ & $4808 \pm 645$ & $3466\pm 736$ 	\\

ASW &  $4533 \pm 248$  & ${\pmb{3208 \pm 16}}$  & $ 44\pm 5$  & ${\pmb{8420 \pm 29}}$  & $1610 \pm 2328$  & $ 602\pm 232$ 	\\

\hline
\bf{Ours}&  ${\pmb{4810 \pm 111}}$    & $3125 \pm 159$  & ${\pmb{111 \pm 22}}$  & $7458 \pm 191$  & ${\pmb{5394 \pm 451}}$ & ${\pmb{4717 \pm 346}}$ 	\\

\bottomrule
\end{tabular}
\end{sc}
\end{center}
\end{table*}

We conduct ablation studies to empirically examine the properties of our AEAIL algorithm. 
In the first ablation, we aime to demonstrate the effectiveness of the encoding-decoding process by using Jensen-Shannon divergence as the distance metric along with a discriminator-based reward function (JSD). This setting mirrored the popular imitation learning method, i.e., GAIL. 

Next, we explore the impact of the distance metric on performance, drawing inspiration from another imitation learning method, i.e., FAIRL. We employe forward KL divergence (F-KLD) as the distance metric for this ablation, considering that it could potentially contribute to improved performance. 

Recognizing that the distance metric could be replaced with a stationary metric, we further investigated the use of PWIL's greedy optimal transport (GOT) as a distance metric for our ablation study.

We also considered the role of the absorbing state wrapper (ASW), as proposed in DAC \cite{DAC}, which could potentially enhance the imitation performance of our AEAIL. Thus, we included ASW (DAC) as another ablation in our study. 

Moreover, it is worth noting that our AEAIL algorithm features two variants: one based on the W-distance metric and another based on JSD. Both variants solely rely on the encoding-decoding process, without utilizing a discriminator-based reward function.

By conducting these ablation studies, we aime to gain insights into the empirical properties of our AEAIL algorithm and understand the impact of various components and metrics on its performance.

\textbf{Question 1.} \emph{Does our AEAIL achieve best performance compared to these four ablation methods?}

Table \ref{final_comparison_non_noisy} (training curves see in Figure \ref{NonNoisyCurve}) shows the final policy performance of different ablation methods on clean expert data. Results show that GOT performs well on Walker2d, Hopper, HalfCheetah, and Humanoid but still not good enough. JSD method and F-KLD are comparable on Hopper, Swimmer, HalfCheetah, while JSD is better than F-KLD on other tasks. ASW achieves great performance on Walker2d, Hopper and HalfCheetah while it performs poorly on other tasks.

Our method's overall scaled reward is about $0.921$, whereas the best ablation method is $0.83$ for JSD. There is an about $11\%$ relative improvement. Our method outperforms other ablations on all locomotion tasks except for Hopper and HalfCheetah. Here we would like to point out that our AEAIL has already achieved $97.8\%$ of the expert performance on Hopper while $91.7\%$ on HalfCheetah, which is very close to completely solving the tasks.

\textbf{Question 2.} \emph{Is there any pattern between the performance gains and the dimension of tasks?} 
\begin{figure*}[htb]
	\centering
	\begin{minipage}{0.3\linewidth}\centering
		\includegraphics[width=\linewidth]{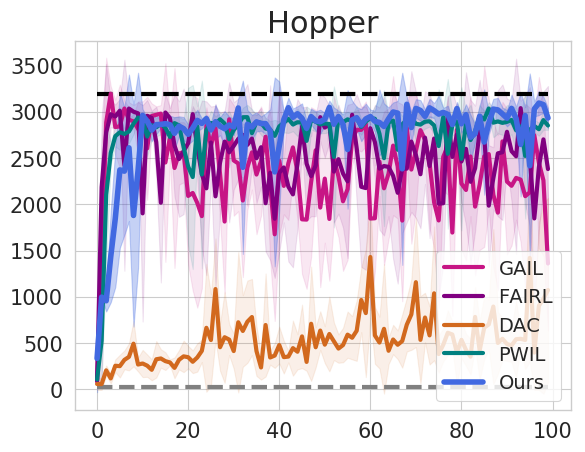}\\
	\end{minipage} 
	\begin{minipage}{0.3\linewidth}\centering
		\includegraphics[width=\linewidth]{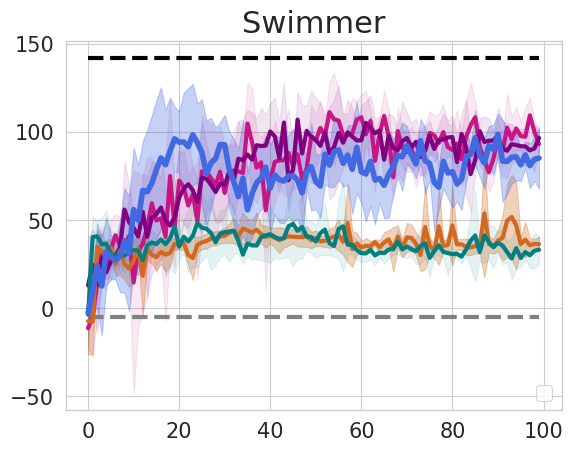}\\
	\end{minipage} 
	\begin{minipage}{0.3\linewidth}\centering
		\includegraphics[width=\linewidth]{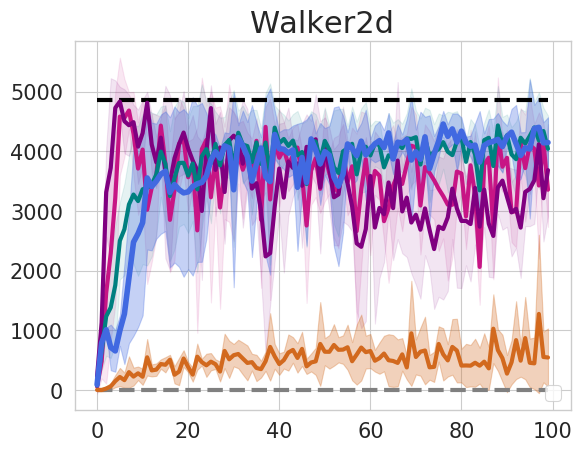}\\
	\end{minipage} 	
 
	\begin{minipage}{0.3\linewidth}\centering
		\includegraphics[width=\linewidth]{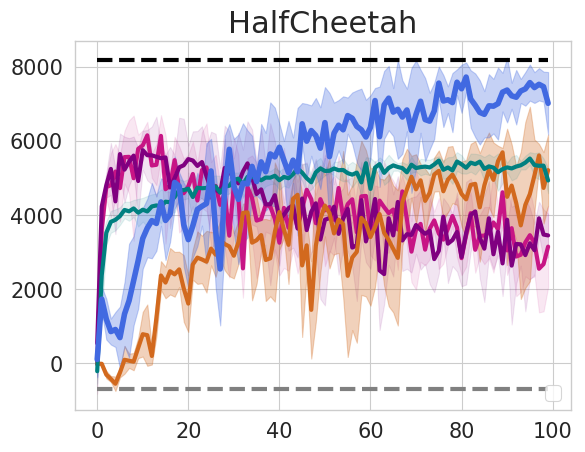}\\
	\end{minipage} 
	\begin{minipage}{0.3\linewidth}\centering
		\includegraphics[width=\linewidth]{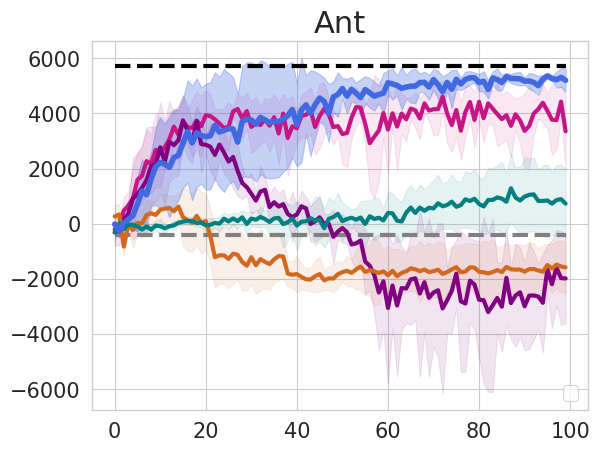}\\
	\end{minipage} 
	\begin{minipage}{0.3\linewidth}\centering
		\includegraphics[width=\linewidth]{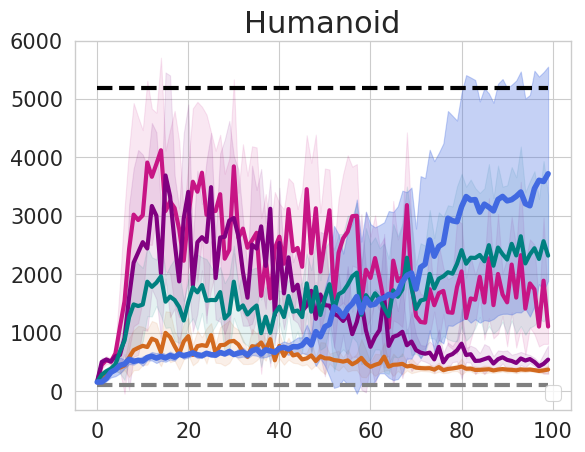}\\
	\end{minipage} 
	\caption{Mean and standard deviation return of the deterministic evaluation policy over 10 rollouts and 5 seeds learning from noisy expert demonstrations, reported every 100k timesteps.}
	\label{NoisyCurve}
\end{figure*}

\begin{figure}
\centering
\includegraphics[width=0.36\textwidth]{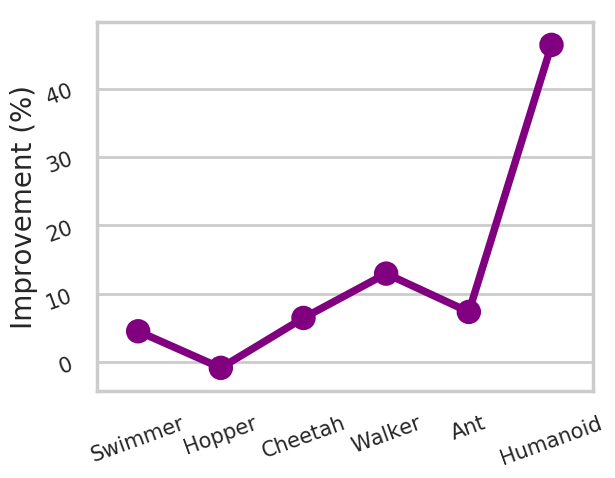}\\
	\caption{Relative improvement of our AEAIL compared to JSD with discriminator-based rewards in different environments. Environments are sorted by the number of state-action dimensions. }
	\label{Improve}
\end{figure}

Figure \ref{Improve} shows the relative improvement of overall scaled rewards for our AEAIL compared to JSD with discriminator-based rewards on clean expert data. The abscissa axis represents different tasks with the growth of task dimension, from Swimmer (dim=10) to Humanoid (dim=393). With the increasing task dimension, the relative improvement of our AEAIL shows an increasing trend. 
It indicates that our method has much more benefits when applied into high dimensional environments. We hypothesize that the discriminator in AIL would focus much more on the minor differences between the expert and agent samples with the growth of task dimension. It is harder for AIL to scale in higher dimensional tasks. On the other hand, our encoding-decoding process provides more information to the policy learning. Consequently, AEAIL scales relatively better than JSD in higher dimensional tasks.

\textbf{Question 3.} \emph{Is our AEAIL robust to the noisy expert data?}

To show the robustness of our proposed AEAIL, we further conduct experiments on noisy expert data. We add a Gaussian noise distribution $(0,0.3)$ to the expert data for Walker2d, Hopper, Swimmer and HalfCheetah. Since Ant and Humanoid are much more sensitive to noise, we add $(0,0.03)$ Gaussian noise to these two tasks. 

Table \ref{final_comparison_noisy} (training curves see in Figure \ref{NoisyCurve}) shows the final learned policy performances on benchmarks. 
These results show that our method outperforms other ablations on all tasks except for Swimmer, on which F-KLD wins. 
Our AEAIL offers an excellent capability in learning from noisy expert data on these tasks. The overall scaled rewards for our AEAIL is $0.813$, whereas the best ablation is $0.539$ for GOT on the noisy expert setting. There is an about $50.7\%$ relative improvement. Other discriminator based ablations, are very sensitive to the noisy expert. 


\begin{table*}
  \caption{Learned policy performance for different imitation learning algorithms on noisy expert data.}
  \label{final_comparison_noisy}
  \centering
  \begin{sc}
  \begin{tabular}{lcccccr}
    \toprule
Task &  Walker2d & Hopper & Swimmer & HalfCheetah & Ant & Humanoid   \\
\midrule

JSD &  $3358 \pm 623$   & $1361 \pm 725$ & $93 \pm 5$ & $3148 \pm 1096$ & $3353 \pm 1370 $ & $ 1110\pm 296$ 	\\

F-KLD  &  $ 3686 \pm 772$   & $ 2483\pm 899$ & $ {\pmb{97\pm 6}}$ & $ 3451\pm 588$ & $-1987 \pm 1633 $ & $ 544\pm 241$ 	\\

GOT  & $4043\pm 559$   & $2854 \pm 48$  & $33 \pm 8$ & $4933 \pm 343$ & $729 \pm 1274$ & $2321\pm 1038$	\\

ASW  &  $529 \pm 138$  & $357 \pm 171$  & $ 36\pm 3$  & $2019 \pm 900$  & $-1586 \pm 922 $  & $ 374\pm 66$	\\
\hline

\bf{Ours}&  ${\pmb{4149 \pm 423}}$    & ${\pmb{2935 \pm 110}}$  & $85 \pm 17$  & ${\pmb{7017 \pm 844}}$  & ${\pmb{5204 \pm 421}} $ & ${\pmb{3721 \pm 1831}}$ 	\\

\bottomrule
  \end{tabular}
  \end{sc}
\end{table*}

\textbf{Question 4.} \emph{How are the latent representations distributed with different reward function formulations?}

\begin{figure*}[t!]
	\centering
	\includegraphics[width=0.92\textwidth]{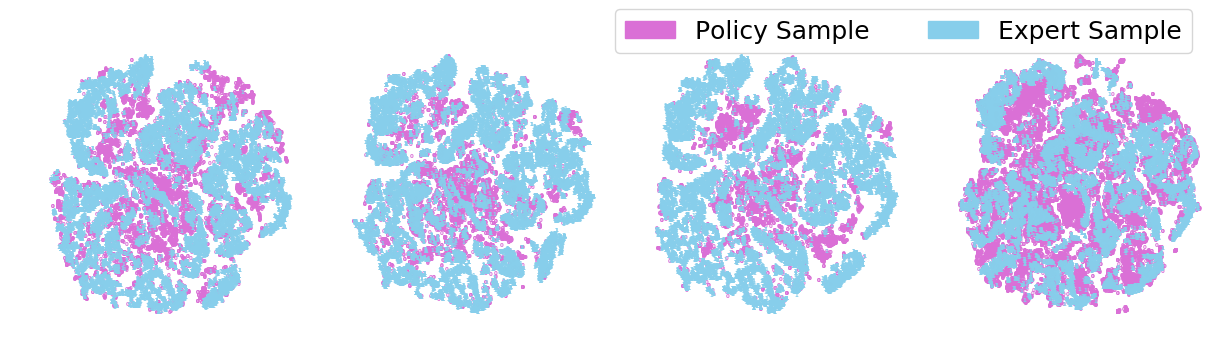}
Discriminator+JSD \qquad \qquad  Discriminator+ASW \qquad \qquad  Discriminator+F-KLD \qquad \qquad \qquad    Our AEAIL
	\caption{t-SNE visualization for the latent space of discriminator or our auto-encoder on noisy expert data on Walker2d. t-SNE's hyper-parameter perplexity is $30.0$.}
	\label{tSNE}
    \vskip -0.1in
\end{figure*}

AIL can be seen as inducing the expert policy through confusing the discriminator. Ideally, the final policy generated samples and expert samples would demonstrate identical distribution. We visualize the latent representations of the discriminator (in AIL methods) and the auto-encoder (in our method) on both the expert and the policy-generated samples with t-SNE. 

Ideally, the expert and final-policy-generated samples will be indistinguishable (greatly overlapping) in this latent space. Specifically, we visualize the output of the discriminator's first hidden (middle) layer for the AIL methods. For our AEAIL, we visualize the output of the auto-encoder's first hidden (middle) layer.

However, since the discriminator focus on the minor differences between the expert and generated samples, the embeddings won't fully overlapped. Figure \ref{tSNE} shows that our AEAIL greatly overlaps the two sets of embeddings compared to other discriminator based formulations with noisy expert data. It reflects that our auto-encoder can't distinguish the expert and generated samples any more in the latent space. It means that our auto-encoder can learn the full-scale difference between expert and generated samples. 

\textbf{Question 5.} \emph{What is the major contributing factor of our AEAIL? Could it be the specific W-distance metric?}

To analyze the major contributing factor of our AEAIL, we conduct an ablation study that replaces the distance to other distribution divergences. Comparable performances indicate that the major contributing factor is the encoding-decoding process rather than the specific distribution divergence. To justify this hypothesis, we formulate another form of surrogate reward function based on Jensen-Shannon divergence \citep{GAN2014}. Details see in Appendix \ref{Appendix:ablation-distance}. 

Jensen-Shannon divergence \citep{GAN2014} is defined as:
\begin{align}\label{JS-distance}
\text{JS}(\text{P}_r, \text{P}_g) = &{E}_{x \sim \text{P}_r}[\log(\frac{\text{P}_r(x)}{\frac{1}{2}(\text{P}_r(x)+\text{P}_g(x))})] +\\
&{E}_{x\sim \text{P}_g}[\log(\frac{\text{P}_g(x)}{\frac{1}{2}(\text{P}_r(x)+\text{P}_g(x))})]
\end{align}
where $\text{P}_r$ and $\text{P}_g$ represent the expert data distribution and the agent-generated sample distribution accordingly. 

Our JS-based variant induces the expert policy by minimizing the distribution distance between expert and generated samples. Specifically, we define that:
\begin{align}
    \text{P}_r(x)/(\text{P}_r(x)+\text{P}_g(x)) = \exp(-\text{AE}(x))
\end{align}
where $\text{AE}(x)$ is the same with Equation \ref{AE} in our original AEAIL.
Minimizing Jensen-Shannon divergence assigns low mean square errors to the expert samples and high mean square errors to other regions. 

\begin{table}
\caption{Relative improvements for different variants of our AEAIL compared to the best baseline JSD and GOT on clean and noisy data, respectively.}
\label{table:ablation}
\begin{center}
\begin{tabular}{llll}
\toprule
Improvements &  Ours  & Ours-JS  & Ours-VAE   \\
\midrule
Clean Data & $11.0\%$ & $10.5\%$  & $7.6\%$   \\

Noisy Data & $50.7\%$ & $44.9\%$  & $42.1\%$  \\

\bottomrule
\end{tabular}
\end{center}
\end{table}

Table \ref{table:ablation} shows that our JS-based variant achieves $10.5\%$ and $44.9\%$ relative improvement compared to the best baseline JSD method and GOT method on clean and noisy expert data, respectively. Similar to the original AEAIL, our JS-based variant also greatly improves the imitation performance on these benchmarks. The relative improvements are comparable between the two distance metrics. It indicates that the major contributing factor of our AEAIL is the encoding-decoding process. This also shows that AEAIL works not limited to a specific distance metric.

\textbf{Question 6.} \emph{Is AEAIL limited to the specific type of auto-encoders? How about utilizing variational auto-encoders?}

To prove that our AEAIL is not sensitive to a specific type of auto-encoders, we conduct experiments by replacing the vanilla auto-encoder with the variational auto-encoder. If variational auto-encoder leads to similar performance on the benchmarks, it illustrates that AEAIL works under a wide range of different auto-encoders. 
We run the VAE-based variant on MuJoCo tasks under the same setting as the AE-based one. Details see in Appendix \ref{Appendix:ablation-vae}.

\begin{figure}
\centering
\includegraphics[width=0.38\textwidth]{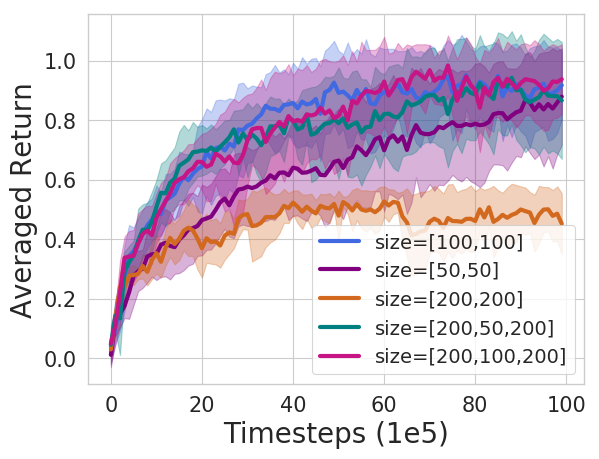}\\
	\caption{Training curves with averaged scaled rewards overall for varying hidden sizes of the AE.} 
	\label{ablationAE}
    \vskip -0.2in
\end{figure}

Table \ref{table:ablation} shows our VAE-based variant gets $7.6\%$ and $42.1\%$ relative improvement compared to the best baseline JSD and GOT on clean and noisy expert data, respectively. This means that our VAE-based variant improves the imitation performance considerably compared to other baselines. It justifies that our AEAIL is flexible with different auto-encoders.

\textbf{Question 7.} \emph{How our AEAIL performs with varying hidden size of the AE?}

To justify that our AEAIL is not sensitive to different sizes of the auto-encoder, we conduct ablations with different hidden sizes of the auto-encoder. If different sizes of the AE lead to similar imitation performance, it proofs that our AEAIL is not sensitive to the hidden size of the auto-encoder.  

Figure \ref{ablationAE} shows the averaged overall scaled rewards on six tasks of AEAIL with different hidden sizes of the auto-encoder. They achieve comparable results for most of the choices. However, when the size of AE is too large, the imitation performance of our AEAIL would decline. It depicts that the hidden size of auto-encoder cannot be too large. Otherwise, the AEAIL is not sensitive to the size of the auto-encoder. 

\section{Conclusions}
This paper presents a straightforward and robust adversarial imitation learning method based on auto-encoding (AEAIL). We utilize the reconstruction error of an auto-encoder as the surrogate pseudo-reward function for reinforcement learning. The advantage is that the auto-encoder based reward function focused on the full-scale differences between the expert and generated samples, which provides a denser reward signal to the agent. As a result, it enables the agents to learn better. Experimental results show that our methods achieve strong competitive performances on both clean and noisy expert data. In the future, we want to further investigate our approach in more realistic scenarios, such as autonomous driving and robotics.



\begin{acks}
This work is supported by the Ministry of Science and Technology of the People's Republic of China, the 2030 Innovation Megaprojects "Program on New Generation Artificial Intelligence" (Grant No. 2021AAA0150000). This work is also supported by the National Key R$\&$D Program of China (2022ZD0161700).
\end{acks}



\bibliographystyle{ACM-Reference-Format} 
\bibliography{aamas}

\onecolumn

\section{Appendix}

\subsection{Proof on that our reward function is K-Lipschitz}\label{Appendix:proof-theorem}

\begin{theorem}\label{theorem1} (the reward function is $2m(K+1)$-Lipschitz if the auto-encoder output is K-Lipschitz)
Provided that $f(x)$ is K-Lipschitz, then the reward function $1/(1+(x-f(x))^2)$ is 2m(K+1)-Lipschitz. 
\end{theorem}

\begin{proof}
~\\

\textbf{Step 1 }
If $f(x)$ is K-Lipschitz, then according to the definition of K-Lipschitz. There exists a constant $K$, for all real $x_1$ and $x_2$, such that:
\begin{align*}
    |f(x_1) - f(x_2)| \leq K |x_1 - x_2|
\end{align*}
where $f(x)$ is a real-valued function.

Then, we can obtain:
\begin{align*}
    |(x_1 - f(x_1) - (x_2 - f(x_2))| = |(x_1-x_2)-(f(x_1)-f(x_2))| \leq (K+1) |x_1 - x_2| 
\end{align*}

So there exist $K_1 = K+1$, for all real $x_1$ and $x_2$, such that
\begin{align*}
    |(x_1 - f(x_1) - (x_2 - f(x_2))| \leq K_1 |x_1 - x_2| 
\end{align*}

We can obtain that $(x - f(x))$ is $K_1$-Lipschitz. 

~\\

\textbf{Step 2 }
Let $g(x) = x - f(x)$, we can obtain that $g(x)$ is $K_1$-Lipschitz, then there exists a constant $K_1$, for all real $x_1$ and $x_2$, such that:
\begin{align}\label{def_lipschitz}
    |g(x_1) - g(x_2)| \leq K_1 |x_1-x_2|
\end{align}

In our AEAIL, $x$ is the input feature which is bounded while $f(x)$ is also bounded since the weights of the auto-encoder is bounded with weight clipping. So $g(x)$ is bounded. 

Therefore, we can obtain:
\begin{align*}
    |g^2(x_1)-g^2(x_2)| = |g(x_1) - g(x_2)| \cdot |g(x_1) + g(x_2)| \leq K_1|x_1-x_2| \cdot |g(x_1) + g(x_2)|
\end{align*}

We assume that $g(x)$ is bounded by a constant $m$, then:
\begin{align*}
    |g^2(x_1)-g^2(x_2)| \leq 2mK_1|x_1-x_2| 
\end{align*}

So there exist $K_2=2mK_1$, such that for all real $x_1$ and $x_2$:
\begin{align*}
    |g^2(x_1)-g^2(x_2)| \leq K_2|x_1-x_2| = 2m(K+1)|x_1-x_2|
\end{align*}

We can obtain that the reconstruction error $g^2(x) = (x-f(x))^2$ is also $2m(K+1)$-Lipschitz. 

~\\

\textbf{Step 3 }
We prove that if $h(x)$ is K-Lipschitz, then $1/(1+h(x))$ is also K-Lipschitz.

Since $h(x)$ is K-Lipschitz, we can obtain that there exists $K$, for all real $x_1$, $x_2$, that

$$|h(x_1) - h(x_2)| \leq K |x_1-x_2|$$

We can obtain:

$$|1/(1+h(x_1)) - 1/(1+h(x_2))| = |(h(x_2)-h(x_1))/(h(x_1+h(x_2)+1+h(x_1)\cdot h(x_2)))|$$

Since we view $h(x)$ as the reconstruction error of an auto-encoder, $h(x) \geq 0$
So We can obtain:

\begin{align}
|1/(1+h(x_1)) - 1/(1+h(x_2))| = &|(h(x_2)-h(x_1))/(h(x_1+h(x_2)+1+h(x_1)\cdot h(x_2)))| \\
&\leq |h(x_2)-h(x_1)| \leq K |x_1-x_2|
\end{align}

So $1/(1+h(x))$ is also K-Lipschitz. 

~\\

\textbf{Conclusion }
So far, we can obtain that if $f(x)$ is K-Lipschitz, then our AEAIL's reward formulation $1/(1+(x-f(x))^2)$ is $2m(K+1)$-Lipschitz.

\end{proof}

\subsection{Proof on that AEAIL is minimizing the W-distance}\label{Appendix:proof-corollary}

\begin{corollary}\label{corollary} (Our AEAIL is minimizing Wasserstein distance) the divergence for AEAIL:
\begin{align}\label{objective}
d(\pi_E, \pi_\theta) = \sup_{|r_w|_L \leq K} {E}_{\pi_E}[r_w(s,a)] - {E}_{\pi_\theta}[r_w(s,a)], 
\end{align}
is Wasserstein distance. 

\end{corollary}

\begin{proof}
\citep{WGAN} shows that given $\mathcal{F}$ is a set of functions from $\mathcal{X}$ to ${R}$, we can define
\begin{align*}
    d_{\mathcal{F}}(P_r, P_{\theta}) = \sup_{f \in \mathcal{F}}{E}_{x\sim P_r}[f(x)] - {E}_{x \sim P_\theta}[f(x)]
\end{align*}
as an integral probability metric with respect to the function class $\mathcal{F}$. When $\mathcal{F}$ is K-Lipschitz functions, we can obtain $d_{\mathcal{F}}(P_r,P_\theta)$ is the Wasserstein distance. 

For our AEAIL, our reward function is:
\begin{align*}
r(s,a) = 1/(1+\Vert x - \text{Dec}\circ \text{Enc}(x)\Vert^2)   
\end{align*}

The reward function is K-Lipschitz when we clip the weight of the auto-encoder. This will yield the same behaviour. Therefore, our AEAIL actually optimizes the W-distance. 

\end{proof}

\subsection{Implementation}\label{Appendix:implementation}

Data, code, and video see in supplementary materials. This section will introduce the implementation details, including tasks, basic settings for all algorithms, baseline implementations, network architecture for our approach. 

\subsubsection{DMControl Tasks}

DMControl is a DeepMind's software stack for physics-based simulation and Reinforcement Learning environments, using MuJoCo physics. It provides both state and image based rendering. In our experiments, we use DMControl in phase 1 experiment. 

\subsubsection{MuJoCo Tasks}

\begin{figure}[htb]
	\centering
	\begin{minipage}{0.16\linewidth}\centering
		\includegraphics[width=0.95\textwidth]{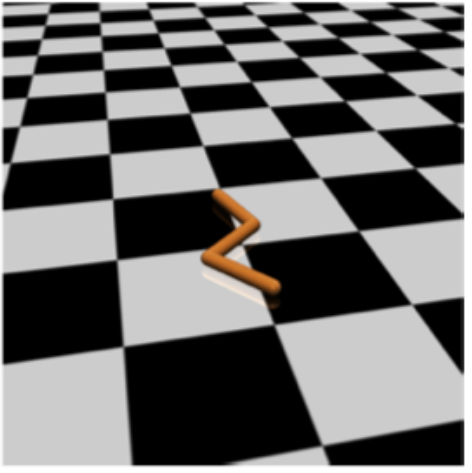}\\
		(a)\label{a}
	\end{minipage} 
	\begin{minipage}{0.16\linewidth}\centering
		\includegraphics[width=0.95\textwidth]{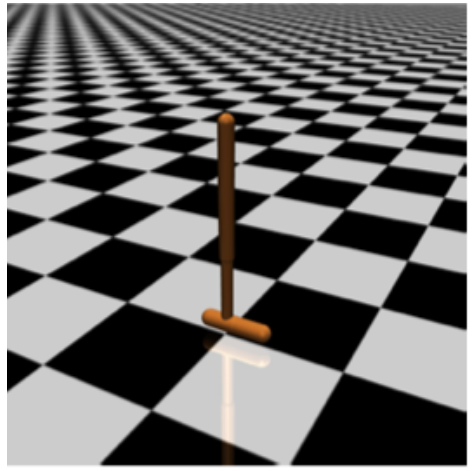}\\
		(b)\label{b}
	\end{minipage} 
	\begin{minipage}{0.16\linewidth}\centering
		\includegraphics[width=0.95\textwidth]{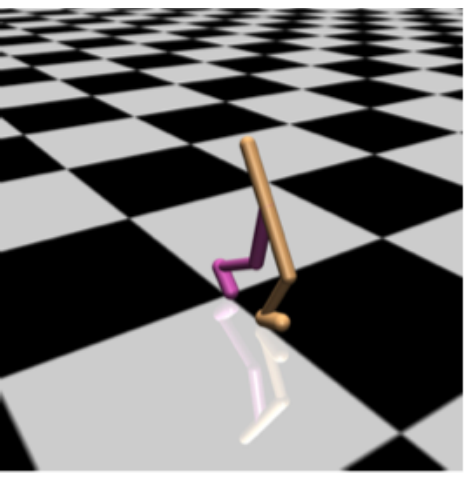}\\
		(c)\label{c}
	\end{minipage}
	\begin{minipage}{0.16\linewidth}\centering
		\includegraphics[width=0.95\textwidth]{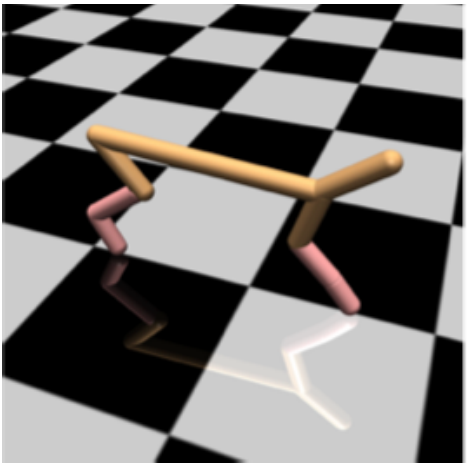}\\
		(d)\label{d}
	\end{minipage} 
	\begin{minipage}{0.16\linewidth}\centering
		\includegraphics[width=0.95\textwidth]{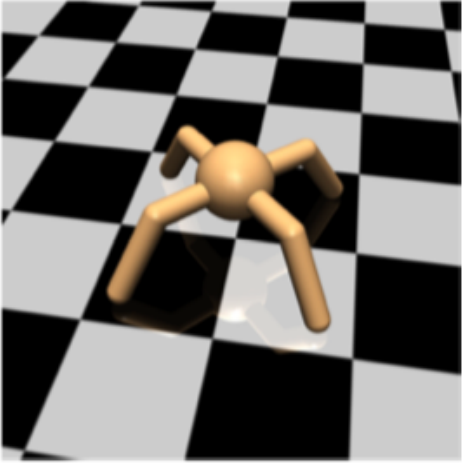}\\
		(e)\label{e}
	\end{minipage} 
	\begin{minipage}{0.16\linewidth}\centering
		\includegraphics[width=0.95\textwidth]{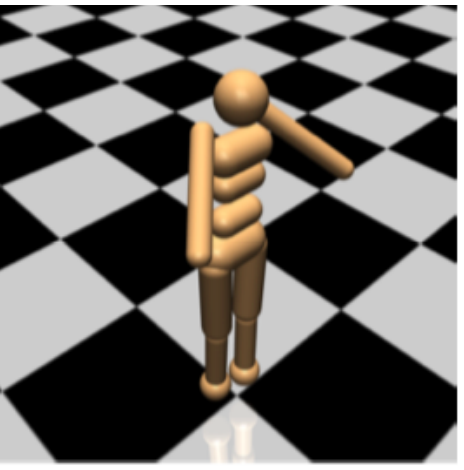}\\
		(f)\label{f}
	\end{minipage}
	\caption{Illustrations for locomotion tasks we used in our experiments: (a) Swimmer; (b) Hopper; (c) Walker; (d) HalfCheetah; (e) Ant; (f) Humanoid.}
	\label{locomotion}
\end{figure}

The goal for all these MuJoCo tasks is to move forward as quickly as possible. These tasks are more challenging than basic ones due to high degrees of freedom. In addition, a great amount of exploration is needed to learn to move forward without getting stuck at local optima. Since we penalize for excessive controls and falling over, during the initial stage of learning, when the robot cannot move forward for a sufficient distance without falling, apparent local optima exist, including staying at the origin or moving forward slowly. Figure \ref{locomotion} depicts locomotion tasks' environments. We use MuJoCo tasks in experiment phase 2. 

\subsubsection{Resources and License}
We use 8 GTX 3080 GPU and 40 CPU to run the experiments. Our algorithm needs $\sim 10$ hours for training. We use an open source MuJoCo license to run the locomotion tasks. The license can be found here: https://mujoco.org.

\subsubsection{Basic settings for Phase 1 Experiments}
Every dataset for DMControl tasks includes at maximum 10 trajectories. We run experiments with 1, 4, 7, 10 trajectories. The code and data is open sourced in supplementary materials. 

\subsubsection{Basic settings for Phase 2 Experiments}
We normalize the state features for each task and use the normalized features as the input for the auto-encoder. We also use the raw data input to compute the final output of the mean square error. 

We train these algorithms directly on Walker2d, Hopper, and Swimmer, while on HalfCheetah, Ant, and Humanoid, we use BC to pre-train the policy with 10k iterations. This operation is based on OpenAI baselines.
We test all the algorithms on these six locomotion tasks with five random seeds and using ten rollouts to estimate the ground truth rewards to plot the training curves. 

The maximum length for sampled trajectories is $1024$. The discounted factor is $0.995$. Each iteration, we update three times of policies and one time of the reward function (discriminator or auto-encoder). While updating TRPO, we update the critic five times with a learning rate of 2e-4 every iteration. 

We clip the auto-encoder's parameters into $[-0.99,0.99]$ to make the AE K-Lipschitz continuous.

\textbf{JSD (as GAIL)} \citep{GAIL} borrows the GAN \citep{GAN2014} architecture to realize the IRL process, and it is applicable in high dimensional dynamic environments. We implement GAIL with open-source code, OpenAI baselines \citep{baselines}, with best tunned hyperparameters.

\textbf{F-KLD (as FAIRL)} \citep{FAIRL} utilizes the forward KL divergence as the objective, which shows competitive performances on MuJoCo tasks. FAIRL is implemented based on GAIL with corresponding objective function as in \citep{FAIRL}.

\textbf{GOT (as PWIL)} \citep{PWIL} is an imitation learning method based on an offline similarity reward function. It is pretty robust and easy to fine-tune the hyper-parameters of its reward function (only two hyper-parameters). PWIL is implemented based on open source code \citep{PWIL} with TRPO as the reinforcement learning method for a fair comparison.

\textbf{ASW (as DAC)} \citep{DAC} is discriminator-actor-critic algorithm which introduces absorbing states' rewards to improve the imitation performance.

\subsubsection{Architecture of our methods}

\textbf{Policy Net:} two hidden layers with 64 units, with tanh nonlinearities in between, which is the same with baselines. The learning rate we used for updating the policy is $0.01$. 

\textbf{Auto-Encoder Net:} (both vanilla auto-encoder and variational auto-encoder) the number of layers is the same with the architecture of discriminator in GAIL, i.e., four layers. The two hidden layers are with 100 units, with tanh nonlinearities in between, the final output layer is an identity layer. We normalize the state feature after the input layer, and we use the raw state action input features to compute the mean square error. The learning rate for the auto-encoder is 3e-4.

\subsection{Empirical study: different distribution divergences}\label{Appendix:ablation-distance}

We conduct an ablation study with Jensen-Shannon divergence. We run the experiments with a new form of reward function based on Jensen–Shannon divergence \citep{GAN2014} for comparison. 

Jensen-Shannon divergence is defined as: 
\begin{align}
    \text{JS}(\text{P}_r, \text{P}_g) = {E}_{x \sim \text{P}_r}[\log(\frac{\text{P}_r(x)}{\frac{1}{2}(\text{P}_r(x)+\text{P}_g(x))})] + {E}_{x\sim \text{P}_g}[\log(\frac{\text{P}_g(x)}{\frac{1}{2}(\text{P}_r(x)+\text{P}_g(x))})],
\end{align}
where $\text{P}_r$ and $\text{P}_g$ represents the expert data distribution and the generated data distribution. 
The adversarial training procedure of AEAIL will match the state-action distribution of expert and generated samples. 

In this form of reward function, we consider the exponential of negative reconstruction error as: $\text{P}_r(x)/(\text{P}_r(x)+\text{P}_g(x))$. The derived Jensen-Shannon divergence for updating the auto-encoder based reward function is:
\begin{align}\label{objective:JS1}
    \mathcal{L} =& - {E}_{(s,a) \sim D_E}[\log(\exp(-\text{AE}_w(s,a)))] - {E}_{(s,a) \sim \pi}[\log(1 - \exp(-\text{AE}_w(s,a)))]\\
    =& {E}_{(s,a) \sim D_E}[\text{AE}_w(s,a)] - {E}_{(s,a) \sim \pi}[\log(1 - \exp(-\text{AE}_w(s,a)))]
\end{align}
Here, $w$ represents the parameters of the auto-encoder network. 
Its formulation of surrogate reward function is:
\begin{align}
    r_w(s, a) = - \log(1 - \exp(-\text{AE}_w(s,a))).
\end{align}

Minimizing the Jensen-Shannon divergence between state-action distribution of the expert and generated trajectories assigns low mean square errors to the expert samples and high mean square errors to other regions. And intuitively, we give high rewards to the regions with low reconstruction errors and vice versa.

\begin{table*}[htb]
\caption{Learned final policy performance for our AEAIL with W-distance, AE based variant (Ours), JS-distance (Ours-JS), and VAE based variant (Ours-VAE). Top three rows: non-noisy expert; Bottom three rows: noisy expert.}
\label{table:distances}
\begin{center}
\begin{sc}
\begin{tabular}{lcccccr}
\toprule
Task &  Walker2d & Hopper & Swimmer & HalfCheetah & Ant & Humanoid   \\
\midrule

\bf{Ours}&  $4810 \pm 111$    & $3125 \pm 159$  & $111 \pm 22$  & $7458 \pm 191$  & $5394 \pm 451$ & $4717 \pm 346$ 	\\

\textbf{JS}&  $4589 \pm 498^{\color{ForestGreen}{\pmb{\downarrow}}}$    & $3245 \pm 65^{\color{red}{\pmb{\uparrow}}}$  & $123 \pm 3^{\color{red}{\pmb{\uparrow}}}$  & $7901 \pm 226^{\color{red}{\pmb{\uparrow}}}$  & $5873 \pm 512^{\color{red}{\pmb{\uparrow}}}$ & $ 3564\pm 432^{\color{ForestGreen}{\pmb{\downarrow}}}$	\\

\textbf{VAE}&  $4813\pm 345^{\color{ForestGreen}{\pmb{\downarrow}}}$    & $3198\pm 142^{\color{red}{\pmb{\uparrow}}}$  & $ 106\pm 6^{\color{ForestGreen}{\pmb{\downarrow}}} $  & $8042\pm 208^{\color{red}{\pmb{\uparrow}}}$  & $4901 \pm 687^{\color{ForestGreen}{\pmb{\downarrow}}}$ & $ 3981 \pm 494^{\color{ForestGreen}{\pmb{\downarrow}}}$	\\

\hline
\bf{Ours}&  $4149 \pm 423$    & $2935 \pm 110$  & $85 \pm 17$  & $7017 \pm 844$  & $5204 \pm 421$ & $3721 \pm 1831$ 	\\

\textbf{JS}&  $4298 \pm 338^{\color{red}{\pmb{\uparrow}}}$    & $2698 \pm 128^{\color{ForestGreen}{\pmb{\downarrow}}}$  & $81 \pm 4^{\color{ForestGreen}{\pmb{\downarrow}}}$  & $6768 \pm 687^{\color{ForestGreen}{\pmb{\downarrow}}}$  & $5381 \pm 512^{\color{red}{\pmb{\uparrow}}}$ & $ 3121 \pm 1254^{\color{ForestGreen}{\pmb{\downarrow}}}$\\

\textbf{VAE} &  $4354\pm 397^{\color{red}{\pmb{\uparrow}}}$    & $3018 \pm 105^{\color{red}{\pmb{\uparrow}}}$  & $ 78\pm 7^{\color{ForestGreen}{\pmb{\downarrow}}}$  & $6782\pm 781^{\color{ForestGreen}{\pmb{\downarrow}}}$  & $4292 \pm 418^{\color{ForestGreen}{\pmb{\downarrow}}}$ & $ 3082 \pm 1291^{\color{ForestGreen}{\pmb{\downarrow}}}$ 	\\
\bottomrule
\end{tabular}
\end{sc}
\end{center}
\end{table*}

The results in Table \ref{table:distances} show that W-distance based method is comparable with JS-distance based method on almost all tasks. Our JS based variant achieves $10.5\%$ and $44.9\%$ relative improvement compared to the best baseline JSD and GOT on clean and noisy expert data, respectively. It is similar to the W-distance based variant. This indicates that the encoding-decoding process is the major contributing factor in our AEAIL which improves the imitation performance rather than the specific distribution divergence.

\subsection{Empirical study: utilizing other auto-encoders}\label{Appendix:ablation-vae}

We conduct experiments with variational auto-encoder based reward function as an ablation. The variational auto-encoder \citep{VAE} is a directed model that uses learned approximate inference. The critical insight behind training variational auto-encoder is to maximize the variational lower bound $\mathcal{L}(q)$ of the log-likelihood for the training examples \citep{DLBook}:
\begin{align}
\mathcal{L}(q) = & {E}_{z \sim q(z|x)}\log p_{\text{model}}(z,x) + \mathcal{H}(q(z|x)) \label{EntropyVariational} \\
                =& {E}_{z \sim q(z|x)}\log p_{\text{model}}(x|z) - D_{\text{KL}}(q(z|x)\Vert p_{\text{model}}(z)) \label{VariationalEq}\\
                \leq & \log p_{\text{model}}(x)
\end{align}
where $z$ is the latent representation for data points $x$. And $p_{\text{model}}$ is chosen to be a fixed Gaussian distribution $\mathcal{N}(0, I)$. 
Training the variational auto-encoder is still a adversarial process. We notice that the first term in Eq. \ref{VariationalEq} is the reconstruction log-likelihood found in traditional auto-encoders. The objective function for training the variational auto-encoder is:
\begin{align}
\begin{split}
    \mathcal{L} = &{E}_{(s,a) \sim D_E}[r_w(s,a)+D_{\text{KL}}(p(z|(s,a)), p_{\text{model}}(z))] \\
            &\quad   - {E}_{(s,a) \sim \pi_\theta}[r_w(s, a)+ D_{\text{KL}}(p(z|(s,a)), p_{\text{model}}(z))]
\end{split}
\end{align}
where the choice for $r_w(s,a)$ is the same with our AE based one, which is Eq. \ref{RFunction} in Section \ref{Method}. When minimizing this objective, the variational auto-encoder based reward function can still assign high values near the expert demonstrations and low values to other regions. On the other hand, optimizing this objective is restricting the latent distribution for the expert to a fixed Gaussian distribution $p_{\text{model}}(z)$ and maximizing the KL distance between the generated samples and the same fixed Gaussian distribution. 

Table \ref{table:distances} shows the results for our AE based variant and VAE based variant on both clean and noisy expert data. Our VAE based variant gets $7.6\%$ and $42.1\%$ relative improvement compared to the best baseline JSD and GOT, respectively. Comprehensively, our AEAIL works well with the VAE formulation.

\end{document}